\documentclass[twoside]{article}
\usepackage[accepted]{aistats2016}

\usepackage{natbib}
\usepackage{amsthm}
\usepackage{amsmath}
\usepackage{amssymb}
\usepackage{algorithm}
\usepackage{algorithmic}
\usepackage{dsfont}
\usepackage{graphicx}
\usepackage{caption}
\usepackage{subcaption}
\usepackage{xspace}
\usepackage{bold-extra}
\usepackage{url}

\newtheorem{lemma}{Lemma}
\newtheorem{theorem}{Theorem}
\newtheorem{definition}{Definition}

\newcommand{\expwix}{\texttt{Exp3-WIX}\xspace}

\newcommand{\expixb}{\texttt{Exp3-IXb}\xspace}
\newcommand{\expixt}{\texttt{Exp3-IXt}\xspace}

\newcommand{\exph}{\texttt{Exp3}\xspace}
\newcommand{\exphix}{\texttt{Exp3-IX}\xspace}

\newcommand{\ti}{_{t,i}}
\newcommand{\tj}{_{t,j}}
\newcommand{\ji}{_{j,i}}
\newcommand{\Ii}{_{I_t,i}}
\newcommand{\hp}{\hat{p}}
\newcommand{\noise}{\xi}
\newcommand{\gweight}{s}

\newcommand{\loss}{\ell}
\newcommand{\oloss}{c}
\newcommand{\hloss}{\wh{\ell}}
\newcommand{\hL}{\widehat{L}}

\newcommand{\II}[1]{\mathbb{I}_{\ev{#1}}}

\newcommand{\R}{\mathbb{R}}
\newcommand{\EE}[1]{\mathbb{E}\left[#1\right]}
\newcommand{\EEcc}[2]{\mathbb{E}\left[\left.#1\right|#2\right]}
\newcommand{\ev}[1]{\left\{#1\right\}}
\newcommand{\pa}[1]{\left(#1\right)}
\newcommand{\bpa}[1]{\bigl(#1\bigr)}

\newcommand{\F}{\mathcal{F}}
\newcommand{\OO}{\mathcal{O}}
\newcommand{\wt}{\widetilde}
\newcommand{\wh}{\widehat}
\newcommand{\tOO}{\wt{\OO}}
\newcommand{\ra}{\rightarrow}

\newcommand{\sumT}{\sum_{t = 1}^T}
\newcommand{\sumj}{\sum_{j = 1}^N}
\newcommand{\sumi}{\sum_{i = 1}^N}

\newcommand{\avgalpha}{\alpha^*_{\text{avg}}}

\newcommand{\bp}{\boldsymbol{p}}

\usepackage[hidelinks]{hyperref}

\begin{document}

\twocolumn[

\aistatstitle{Online learning with noisy side observations}

\aistatsauthor{ Tom\'a\v s Koc\'ak \And Gergely Neu \And Michal Valko }

\aistatsaddress{ SequeL team \\INRIA Lille - Nord Europe \And Universitat Pompeu Fabra\\ Barcelona, Spain \And SequeL team \\ INRIA Lille - Nord Europe} ]

\begin{abstract}
We propose a new partial-observability model for online learning problems where the learner, besides its own
loss, also observes some \emph{noisy} feedback about the other actions, depending on the underlying structure of
the problem. We represent this structure by a weighted directed graph, where the edge weights are related to the
quality of the feedback shared by the connected nodes. Our main contribution is an efficient algorithm
that guarantees a regret of $\tOO(\sqrt{\alpha^* T})$ after $T$ rounds, where $\alpha^*$ is a novel  graph  property
that we call the \emph{effective independence number}. Our algorithm is completely parameter-free
and does not require knowledge (or even estimation) of $\alpha^*$. For the special case of binary edge weights, our
setting reduces to the partial-observability models of \citet{mannor2011from} and \citet{alon2013from} and our 
algorithm recovers the near-optimal regret bounds.
\end{abstract}



\section{Introduction}
The general framework of online learning considers sequential decision-making problems where a \emph{learner} repeatedly
chooses actions so as to minimize the sum of losses assigned by the \emph{environment} in response to the learner's
actions. 
After making each decision, the learner observes some feedback about the losses assigned by the environment.
Traditionally, the literature considers two types of feedback: \emph{full-information} feedback \citep{cesa-bianchi2006prediction}, 
where the learner observes the losses associated with \emph{all} potential decisions and \emph{bandit} 
feedback \citep{auer2002finite}  where the learner only observes the loss of its own decisions. More recently, 
\citet{mannor2011from} proposed a partial-feedback scheme that models situations that lie between the two extremes: in 
their model, the learner observes losses associated with some additional actions besides its own loss. While this 
framework is often more realistic than either of the two extremes, it fails to address one important practical concern: 
in reality, one can rarely expect \emph{perfect} side observations to be available. In the current paper, we propose a 
similar model that can incorporate \emph{imperfect} side observations corrupted by various levels of noise, depending on 
the problem structure.



As an illustration to our setting, consider the problem of controlling solar panels so as to maximize their power
production. In this problem, the learner has to repeatedly decide about the orientation of the
panels so as to find alignments with strong sunshine. Besides the amount of the energy
being actually produced in the current alignment, the learner can also possibly base its decisions on measurements of
sensors installed on the solar panel. However, the observations
generated by these sensors can be of variable quality depending on visibility conditions, the quality of the sensors and
the alignment of the panels. Overall, this problem can be seen as a bandit problem with noisy side observations fitting
into our framework, where actions correspond to alignments and the noisy side observations give 
information about similar alignments.


%

Intuitively, in the case when the noise level of side observations does not change with time, 
a possible strategy one can think of is to use only the observations from the ``most reliable'' sources and ignore
the rest. Having made the distinction between ``reliable'' and ``unreliable'', the learner could model the observation
structure in the framework of \cite{mannor2011from,alon2013from}, by treating every ``reliable'' observation
as \emph{perfect}. This approach raises two concerns. First, determining the cutoff for unreliable observations that
allows the ``most efficient'' use of information is a highly nontrivial design choice. 
As we show later, knowing the \emph{perfect cutoff} would help us to improve performance over the
pure bandit setting without side observations. Second, one has to address the \emph{bias} arising from handling every
reliable observation as perfect. While one can think of many obvious ways to handle this bias by appropriate weighting
observations, none of these solutions are directly compatible with the model of \cite{mannor2011from,alon2013from}.
Our main contribution in this paper is an algorithm that is able to deal with both issues \emph{without
the knowledge of the optimal cutoff.}

The main tool we use for modeling uncertain observations is a \emph{weighted directed graph} encoding the
quality of side observations. In this graph, the weight of the arc $i\rightarrow j$ measures the quality of the
side observation obtained from action $j$ when selecting action $i$. All weights are assumed to lie in the interval 
$[0,1]$, with a weight of $1$ corresponding to a perfectly accurate side observation, and a weight of $0$ corresponding 
to a side observation of useless noise. Our model generalizes the previously considered models of 
\citet{mannor2011from} and \citet{alon2013from}: their
respective settings are captured by considering undirected and directed graphs with binary weights in our framework. 
In these special cases, the \emph{independence number}\footnote{The independence number of a graph $G$ is defined as
the size of the largest set of points in the graph such that no two points within this set are connected.} $\alpha$ of
the observation graph plays a key role in characterizing the complexity of learning: the minimax regret after $T$
rounds is known to be $\wt\Theta(\sqrt{\alpha T})$. In this paper, we define a similar quantity for weighted graphs:
the \emph{effective independence number} $\alpha^*$ and propose a learning algorithm that enjoys a regret
bound of $\tOO(\sqrt{\alpha^* T})$ without any conditions made on the loss sequence. 
The effective independence number $\alpha^*$ is closely related to the cutoff threshold for noisy observations. 
Intuitively, it is linked to the independence number of a graph that only considers reliable observations. 
In practical scenarios, neither the cutoff nor $\alpha^*$ is ever known to the learner, which is the \emph{main challenge} we need to address. 
In any case, the most interesting situations for our setting are the cases when we can 
bound  $\alpha^*$ by a small quantity.

While we are mainly inspired by situations where the weights of the graph
are fixed and known in advance, we treat a more general setting where the observation structure can arbitrarily change
over time and the weights are revealed to the learner only after it has made its decision. Our algorithm is fully
adaptive in the sense that it does not require any prior knowledge of the sequence of observation graphs or the time
horizon. To achieve this result, we combine the \emph{implicit exploration} strategy
introduced by \citet{kocak2014efficient} with a loss estimation technique that effectively suppresses the observation
noise.
%
%
For the special case of binary weights, the effective independence number and the independence
number
coincide; otherwise $\alpha^*$ is bounded by the number of actions $N$. Thus, the regret bound of our algorithm is
of near-optimal order for binary graphs and is always within logarithmic factors of the minimax regret of order
$\sqrt{NT}$
for the standard multi-armed bandit problem without side observations. As we will show 
later in the paper, there are several interesting cases for which the effective independence number can be bounded in a 
nontrivial way.

Independently of the work presented in this paper, \citet*{wu2015online} considered an essentially identical 
partial-observability model for online learning: there, side observations are modeled as zero-mean Gaussian random 
variables with \emph{variance} depending on the chosen action. It is easy to see that their model and ours can capture 
exactly the same type of problems: a side observation with zero variance in their model corresponds to a perfect 
observation with weight one in our model while useless noise is equivalently represented by infinite-variance or 
zero-weight observations. The results of \citet{wu2015online} are, however, of a completely different flavor than the ones 
presented in the current paper; the primary difference being that \citeauthor{wu2015online} assume that the losses are 
i.i.d.~Gaussian random variables while our results hold without any assumptions made on the sequence of losses. The 
main contributions of \citeauthor{wu2015online}\@ are (i) a general problem-dependent lower bound on the regret and (ii) 
algorithms that work under the assumption that all the useful (i.e., finite-variance) side observations have the same 
variance. This latter assumption does not use the full strength of the framework  where the variance of side 
observations can vary for different actions. Notably, the regret bounds presented in our paper match (up to 
logarithmic factors) the lower bounds of \citet{wu2015online} for the special cases that they consider. That said, their lower 
bounds and our upper bounds are not directly comparable for more general observability graphs.

Besides the works mentioned above, several other partial-observability models have been considered in the literature.
The most general of these settings is the \emph{partial-monitoring} framework considered by \citet{bartok2011minimax,bartok2014partial}. 
Unlike our
model, this framework is most useful for identifying and handling feedback structures that are \emph{more restrictive}
than bandit feedback. In contrast, our framework deals with feedback structures that are strictly more expressive than
plain bandit feedback. Similarly to \citeauthor{bartok2011minimax}, the recent work of \citet{alon2015online} also considers a generalization of
the partial-observability models of \citet{mannor2011from} and \citet{alon2013from} that may be more restrictive than 
bandit
feedback. Another well-studied setting in machine learning is where the observations are corrupted by noise
irrespective of the decisions of the learner (see, e.g., \citealp{cesa-bianchi2010online}). Such settings do not pose an
exploration-exploitation dilemma to the learner and thus are not relevant to our goals.\footnote{In fact, it can be
shown by the techniques of \citet{devroye2013random} that in the setting of online learning with finite actions and
observations corrupted by the same level of i.i.d.~noise, the simplest possible strategy of \emph{following the leader}
gives near-optimal guarantees.}





\section{Background}
Let us now give the formal definition of our learning problem. We consider a sequential decision-making problem where a 
\emph{learner} and an \emph{environment} interact in the following way (see also Figure~\ref{fig:protocol}). In every 
round $t\in[T]=\ev{1,2,\dots,T}$, the environment selects a weighted graph $G_t$ with $N$ 
nodes and a loss function $\loss_t:[N]\ra[0,1]$ where $\loss_{t,i}$ is the loss associated with arm $i$. The weight of 
each arc $i\ra j$ in $G_t$ is denoted as $s_{t,(i,j)}$ and assumed to lie in $[0,1]$. 
Following the environment's move, the learner selects an \emph{action} (or \emph{arm}) $I_t\in[N]$ and incurs 
the loss $\loss_{t,I_t}$. Finally, the learner also observes $G_t$ and the feedback 
\[
c_{t,i} = \gweight_{t,(I_t,i)}\cdot \loss_{t,i} + \left(1-\gweight_{t,(I_t,i)}\right)\cdot \noise_{t,i}
\] 
for every arm $i$, where $\noise_{t,i}$ 
is the 
\emph{observation noise}. We assume that each $\noise_{t,i}$ is zero-mean, satisfies $|\noise_{t,i}|\le R$ for some 
known constant $R\ge 0$, and is generated independently of all other noise terms and the history of the 
process\footnote{We are mainly interested in the setting where $R = \Theta(1)$, that is, we are neither in the easy 
case where $R$ is close to zero or the hard one where it may be as large as $\Omega(\sqrt{T})$}. 
The interaction history between the learner and the environment up to the end of round~$t$ is captured by the 
sigma-algebra $\F_t$. In this work, we consider \emph{adaptive} (or \emph{non-oblivious}) environments that are allowed 
to choose $\loss_t$ and $G_t$ in full knowledge of the history $\F_{t-1}$. We also assume that all graphs $G_t$ are 
such that $\gweight_{t,(i,i)}=1$ for all $i$, that is, the learner always observes its own loss $\loss_{t,I_t}$ without 
corruption.

\begin{figure}[t]
 \vspace{-1em}
\centering
\fbox{
\begin{minipage}{.45\textwidth}
{\bfseries Parameters}: \\
\phantom{aa}set of arms $[N]$, number of rounds $T$.\\
{\bfseries For all $t=1,2,\dots,T$ repeat}
\begin{enumerate}
 \item The environment picks a loss function $\loss_t:[N]\ra [0,1]$ and a directed weighted graph $G_t$ with 
edge weights in $[0,1]$.
 \item Based on its previous observations (and possibly some source of randomness), the learner picks an action
$I_t\in[N]$.
 \item The learner suffers loss $\loss_{t,I_t}$.
 \item The learner observes  $G_t$ and the feedback 
 \[
c_{t,i} = s_{t,(I_t,i)}\cdot \loss_{t,i} + \left(1-s_{t,(I_t,i)}\right)\cdot \noise_{t,i}
 \] 
 for every arm $i\in[N]$.
\end{enumerate}
\end{minipage}
}
\caption{The protocol of online learning with noisy observations.}
\label{fig:protocol}
\end{figure}

The goal of the learner is to choose its actions so as to ensure that its cumulative loss grows as slowly as possible. 
As traditional in the online learning literature \citep{cesa-bianchi2006prediction}, we measure the performance of the 
learner in terms 
of the (total expected) \emph{regret} defined as the gap between the expected loss of the player and the expected loss 
of 
the best fixed-arm policy:
\[
R_T = \max_{i\in[N]}\EE{\sumT \ell_{t,I_t} - \sumT\ell\ti}. 
\]
In this paper, we are interested in constructing algorithms for the learner that guarantees a tight upper bound on the 
regret. Before proposing our algorithm, a few comments are in order. First, notice that our framework 
technically contains the settings of \citet{mannor2011from} and \citet{alon2013from} as special cases where the edge 
weights are chosen from $\ev{0,1}$: in this situation, our framework suggests that the learner either gets 
\emph{perfect} side-observations or just zero-mean noise, which can be safely ignored by the learner. Also, notice that 
since we assume $\gweight_{t,(i,i)}=1$ for all $i$, our problem is not harder for the learner than the standard 
multi-armed bandit problem. 
Indeed, thanks to this property, the learner could simply ignore all side observations and run a bandit algorithm such 
as \exph of \citet{auer2002finite} that guarantees a regret bound of $\OO(\sqrt{NT\log N})$.

\section{Algorithms and main result}

This section presents our main contribution: a learning algorithm with strong theoretical performance 
guarantees for the setting described in the previous section. As the intuitions underlying our algorithm are rather 
intricate, we will proceed gradually: we first identify the main challenges of constructing learning algorithms for our 
setting, then offer a solution that overcomes these difficulties in an efficient manner.

A central concept in our performance guarantees is a new graph property that we call \emph{effective independence  
number}, defined as follows:
\begin{definition}
Let $G$ be a weighted directed graph with $N$ nodes and edge weights bounded in $[0,1]$. For all
$\varepsilon\in[0,1]$, let $G(\varepsilon)$ be the (unweighted) directed graph where arc $i\ra j$ is present if and only
if $\gweight_{i,j} \ge \varepsilon$ in $G$. Letting $\alpha(\varepsilon)$ be the independence number of
$G(\varepsilon)$, the effective independence number of $G$ is defined as
\[
 \alpha^* = \min_{\varepsilon\in[0,1]}\frac{\alpha(\varepsilon)}{\varepsilon^2}\cdot
\]
\end{definition}
Roughly speaking, the effective independence number is a measure of connectivity of weighted graphs. A detailed 
discussion of the effective independence number is deferred to Section~\ref{sec:ind}. In what follows, 
we describe two learning algorithms that guarantee a regret bound depending on the effective independence numbers 
$(\alpha^*_t)$ of the observation graphs $(G_t)$ as $\tOO\pa{\sqrt{\sum_{t} \alpha^*_t}}$.

For presenting our ideas (and our eventual algorithm), we take as template the seminal \exph algorithm of 
\citet{auer2002finite}, as presented by \citet{bubeck2012regret} (see Algorithm~\ref{alg:algorithmTemplate}). The main idea 
of this algorithm is maintaining an estimate $\hloss_{t,i}$ of the losses $\ell_{t,i}$ for every $t$ and $i$ and 
choosing arm $i$ with probability proportional to $\exp\bpa{-\eta_t \sum_{s=1}^{t-1}\hloss_{s,i}}$ in round $t$, where 
$\eta_t>0$ is a parameter of the algorithm often called the \emph{learning rate}.
The main challenge in constructing a learning algorithm for our setting is designing appropriate estimates for the 
losses. In particular, it is obvious that the learner should not rely on observations with a high amount of noise in the 
same way as it relies on observations with almost no noise. One natural way to address this issue is explicitly 
distinguishing between ``reliable'' and ``unreliable'' side observations, and using only reliable sources for estimating 
losses. We first show that while this intuitive loss-estimation method does lead to strong performance guarantees, it 
requires a very careful choice of the cutoff parameter distinguishing reliable and unreliable sources. In 
Section~\ref{sec:wix}, we propose our main algorithm that overcomes this issue and guarantees equally strong performance 
guarantees without having to explicitly distinguish between reliable and unreliable sources.

\begin{algorithm}[t]
\caption{Algorithm template: \\  \  \ \exph \citep{auer2002finite}}
\label{alg:algorithmTemplate}
\begin{algorithmic}[1]
\STATE \textbf{Initialization:} $\hL_{0,i} = 0$ for all $i\in[N]$.
\FOR{$t = 1$ {\bfseries to} $T$}
\STATE Set $\eta_t$ and $\gamma_t$.
\STATE Construct the probability distribution $\bp_t$ with.
$$p\ti = \frac{\mbox{exp}\big(-\eta_t\hL_{t-1,i}\big)}{\sumj \mbox{exp}\big(-\eta_t\hL_{t-1,j}\big)}.$$
\STATE Play random arm $I_t$ according to $\bp_t$.
\STATE Incur loss $\ell_{t,I_t}$.
\STATE Observe $c\ti = \gweight_{t,(I_t,i)}\ell\ti + (1-\gweight_{t,(I_t,i)})\noise\ti$ for all $i\in[N]$.
\STATE Observe graph $G_t$.
\STATE Construct loss estimates $\hloss\ti$.
\STATE Set $\hL\ti = \hL_{t-1,i} + \hloss\ti$.
\ENDFOR
\end{algorithmic}
\end{algorithm}

\subsection{A na\"ive algorithm: \expixt}
We first consider an algorithm that bases its decisions on the following estimates of each $\loss_{t,i}$:
\begin{equation}\label{eq:cand1}
\hloss\ti^{\ (\mbox{\textsc{b}})} = \frac{c_{t,i}}{\sumj p\tj s_{t,(j,i)} + \gamma_t },
\end{equation}
where \textsc{b} stands for ``basic''.
Here, $\gamma_t\ge 0$ is a so-called \emph{implicit exploration} (or, in short, IX) parameter first used by 
\citet{kocak2014efficient} for decreasing the variance of importance-weighted estimates. Notice that setting 
$\gamma_t = 0$, makes estimates above unbiased since 
\[
 \EEcc{c_{t,i}}{\F_{t-1}} = \pa{\sum_{j=1}^N p\tj s_{t,(j,i)}}\cdot \loss_{t,i},
\]
where we used our assumption that $\EE{\noise_{t,i}}=0$. Using these estimates in our algorithmic template \exph (see 
Algorithm~\ref{alg:algorithmTemplate}), one would expect to get reasonable performance guarantees. Unfortunately 
however, we were not able to prove a performance guarantee for the resulting algorithm. 

A close examination reveals that the reason for the poor performance of the above algorithm is the large variance of 
the estimates~\eqref{eq:cand1} which is caused by including observations from ``unreliable sources'' with small 
weights. One intuitive idea is to explicitly draw the line between reliable and unreliable sources by cutting 
connections with weights under a certain threshold. This effect is realized by the estimates
\begin{equation}\label{eq:cand2}
\hloss\ti^{\ (\mbox{\textsc{t}})} = \frac{c_{t,i}\II{s_{t,(I_t,i)}\ge \varepsilon_t}}{\sumj p\tj s_{t,(j,i)} 
\II{s_{t,(j,i)}\ge \varepsilon_t} + \gamma_t},
\end{equation}
where $\varepsilon_t\in[0,1]$ is a threshold value and \textsc{t} stands for ``thresholded''. We call the algorithm 
resulting from using the above estimates in Algorithm~\ref{alg:algorithmTemplate} \expixt, standing for ``\exph with 
Implicit eXploration and Truncated side-observation weights''. Thanks to the thresholding operation, the variance of 
the loss estimates can be nicely controlled and it becomes possible to prove a strong performance guarantee for \expixt.
In particular, we prove the following result about the regret of \expixt:
\begin{theorem}
\label{thm:naive}
For all $t$, let $\alpha^*_t$ be the effective independence number of $G_t$. Then, there exists a setting of 
$(\eta_t)$ and $(\gamma_t)$ for which the regret of \expixt is bounded as
\[
R_T = \tOO\pa{(1+R)\sqrt{\sum_{t=1}^T \frac{\alpha(G_t(\varepsilon_t))}{\varepsilon_t^2}}}.
\]
\end{theorem}
The theorem is proved in the Appendix.
Note that if we choose $\varepsilon_t = \arg\min_{\varepsilon\in[0,1]}\frac{\alpha(G_t(\varepsilon))}{\varepsilon^2}$ 
for 
all $t$, the above bound essentially becomes $\tOO(\sqrt{\avgalpha T})$ where $\avgalpha = \frac 1T \sum_{t=1}^T 
\alpha^*_t$ is the average effective independence number of the sequence of graphs played by the environment. Note 
however that tuning $\varepsilon_t$ can be a very challenging task in practice, since computing independence 
numbers in general is known to be NP-hard. Even worse, computing the \emph{effective} independence number 
of a weighted graph can require computing up to $N^2$ independence numbers. In the next section, we propose an adaptive 
algorithm that does not need to tune this parameter and still manages to guarantee the same regret bound.


\subsection{An adaptive algorithm: \expwix}\label{sec:wix}
This section presents our main algorithm that obtains strong regret bounds without having to estimate any effective 
independence numbers. The key element of this algorithm is using loss estimates of the form
\begin{equation}\label{eq:est}
\hloss\ti = \frac{\gweight_{t,(I_t,i)} \cdot c\ti}{\sumj p\tj \gweight_{t,(j,i)}^2 + \gamma_t},
\end{equation} 
where $\gamma_t\ge 0$ is again the so-called implicit exploration parameter already introduced in 
the previous section. Notice that the difference from the estimates~\eqref{eq:cand1} is that the observation $c\ti$ 
is multiplied by the weight of useful information in $c_{t,i}$ and the denominator is modified accordingly, so that the 
estimates are unbiased when setting $\gamma_t = 0$ since
\[
 \EEcc{s_{t,(I_t,i)}\cdot c_{t,i}}{\F_{t-1}} = \pa{\sum_{j=1}^N p\tj s_{t,(j,i)}^2}\cdot \loss_{t,i}.
\]
The role of this scaling is pulling the noise term $\noise\ti$ toward zero for actions $i$ with small weights 
$\gweight\Ii$, and thus achieving a similar variance-reducing effect as the truncations employed by \expixt.

Armed with the loss estimates~\eqref{eq:est}, we are ready to define our algorithm: \exph (presented as 
Algorithm~\ref{alg:algorithmTemplate}) with Weighted observations and Implicit eXploration, or, in short, \expwix. 
Overall, \expwix has two set of parameters to tune: the sequence of learning rates $\pa{\eta_t}_t$ and the sequence of 
IX parameters $\pa{\gamma_t}_t$. Our main theorem below states the performance guarantees of \expwix with an adaptive 
learning-rate sequence that does not need any prior knowledge about the number of rounds or the nature of the 
side-observation graphs. The key quantity for computing the parameters $\eta_t$ and $\gamma_t$ is
\[
 Q_t = \sum_{i=1}^N \frac{p_{t,i}}{\sum_{j=1}^N p_{t,j} s_{t,(j,i)}^2 + \gamma_t},
\]
defined for all $t$. 
\begin{theorem}
\label{thm:mainTheorem}
For all $t$, let $\alpha^*_t$ be the effective independence number of $G_t$. Then, setting 
\[
\eta_t = \sqrt{\frac{{\log N}}{{2(1+R+R^2)(N + \sum_{s=1}^{t-1}Q_s)}}}
\]
and $\gamma_t =R\eta_t$, the regret of \expwix is bounded as
\[
R_T = \tOO\pa{(1+R)\sqrt{N+\sum_{t=1}^T \alpha^*_t}}.
\]
\end{theorem}
The theorem is proved in Section~\ref{sec:analysis}.
In plain words, Theorem~\ref{thm:mainTheorem} guarantees that the regret of \expwix grows as
$\tOO(\sqrt{\avgalpha T})$. Notice  that  in order  to obtain this regret bound, \expwix never needs to compute the 
effective independence number of any of the observation graphs.
This saves us from a significant computational overhead as compared to the na\"ive algorithm \expixt that 
needed to set a truncation parameter to discard unreliable observations.


\section{The effective independence number}\label{sec:ind}
The previous section has established that the performance guarantees of our algorithms can be expressed in terms of the
effective independence number of the observation graphs. In this section, we provide some basic insights about the
nature of this quantity and describe some graph structures with small effective independence numbers. 

\begin{figure*}
 \vspace{-0.2em}
    \centering
    	\begin{subfigure}[t]{0.3\textwidth}
	        \includegraphics[trim=4.5cm 8.5cm 4.5cm 8.5cm, width=\textwidth]{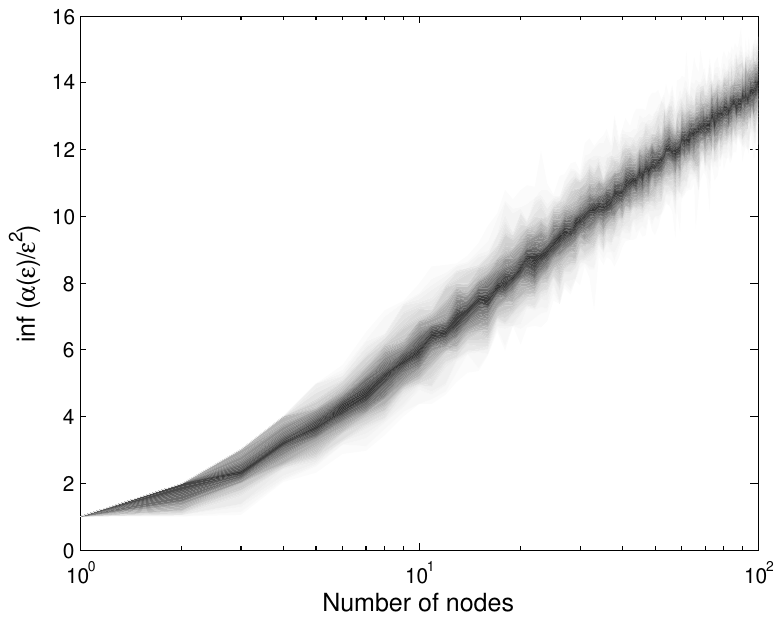}
        	\caption{$U(0,1)$ weights}
	    	\label{fig:U(0,1)weights}
        \end{subfigure}
        \quad
      	\begin{subfigure}[t]{0.3\textwidth}
        	\includegraphics[trim=4.5cm 8.5cm 4.5cm 8.5cm, width=\textwidth]{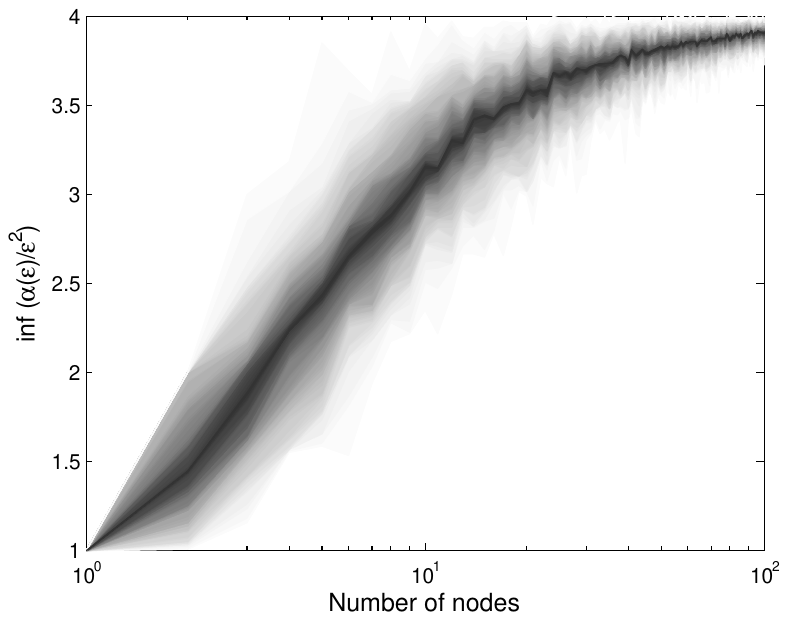}
        	\caption{$U(\frac 12,1)$ weights}
        	\label{fig:U(0.5,1)weights}
        \end{subfigure}
        \quad
        \begin{subfigure}[t]{0.3\textwidth}
        	\includegraphics[trim=4.5cm 8.5cm 4.5cm 8.5cm, width=\textwidth]{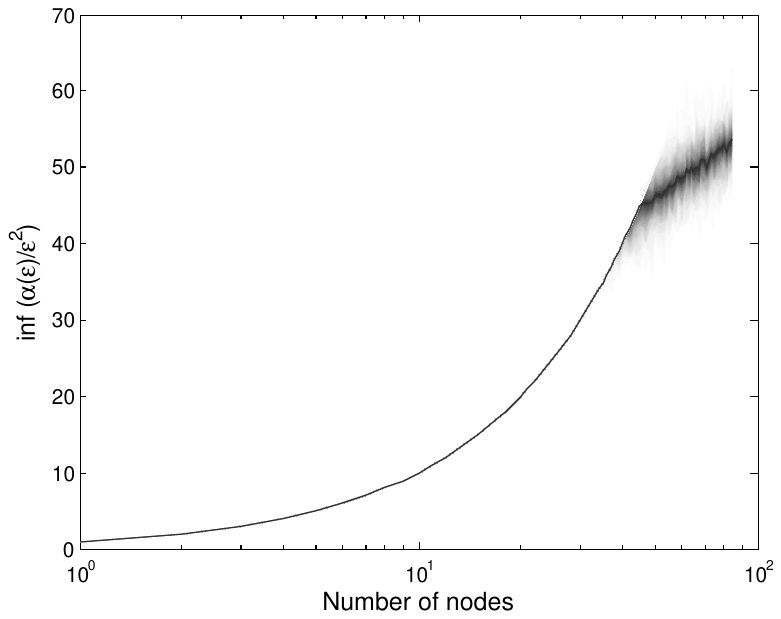}
        	\caption{$U(0,\frac 12)$ weights}
        	\label{fig:U(1,0.5)weights}
        \end{subfigure}
        \caption{Dependence of $\alpha^*$ on the size of the graph with random weights, 100 graphs for each size.}
        \label{fig:randweights}
\end{figure*}

The first observation we make is that the effective independence number is always well-defined, as the 
function $\alpha(\varepsilon)/\varepsilon^2$ can be easily shown to be piecewise decreasing and lower semicontinuous 
with at most $N$ discontinuities. Thanks to these properties, this expression takes its minimum within the closed 
interval $[0,1]$. Second, we note that the effective independence number of any weighted graph is 
trivially bounded by the number $N$ of the nodes in the graph. This follows from the fact that $\alpha^*\le \alpha(1)/1 
\le N$. This essentially guarantees that incorporating side-observations can never be harmful to the performance of the 
learner: the regret  of \expwix is always within logarithmic factors of the minimax regret of order $\sqrt{NT}$ for the
standard multi-armed bandit problem without side observations. 

It is also easy to see that the effective independence number exactly matches the independence number if all edge 
weights are binary. This, in particular, implies that for such graphs, the regret of \expwix grows at the minimax rate 
established by \citet{alon2013from} up to logarithmic factors, matching the performance guarantees of the algorithms of 
\citet{alon2013from} and \citet{kocak2014efficient}. 
Another interesting case is when all weights are either zero or equal to a fixed constant~$\varepsilon$, also assuming 
$s_{i,i}=\varepsilon$. In this case, the effective independence number becomes $\frac{\alpha}{\varepsilon^2}$, 
where $\alpha$ is the independence number of the underlying unweighted graph. This case was studied in the 
recent paper of \citet{wu2015online}, who show (in their Corollary~4) that the \emph{minimax} regret in this case is of
$\Theta(\sqrt{\alpha T}/\varepsilon)$---implying that our performance bounds for this case are again 
near-optimal\footnote{While we prove our bounds for the case where $s_{i,i}=1$ for all $i$, it is easy to extend our 
results to the case where all such weights equal a constant in $[0,1]$.}.
Also, observe that whenever all weights are bounded by some constant $c>0$ from below, the effective independence number 
becomes upper-bounded by $1/c^2$, \emph{irrespective of the number of actions}. That is, our algorithm can achieve an 
\emph{exponential} performance gain over bandit algorithms in terms of $N$ by leveraging such feedback structures.

%

Let us now describe a class of weighted graphs with bounded effective independence numbers. Consider a geometric graph
whose nodes represent vertices of a uniform $k\times k$ grid on $[0,1]^2$. The weight of edge $(i,j)$ is given as
$1/(1+d_{i,j}^2)$, where $d_{i,j}$ is the Euclidean distance of the respective vertices represented by $i$ and $j$.
This graph can be used to model a sensor network where the measurement accuracy of measurements degrades with the
distance. Thus, reading the measurements from one sensor will give information about the measurements of nearby sensors
as well. Intuitively, increasing the number of sensors (i.e., refining the grid) should only improve the
information-sharing between sensors up to a certain level. It is natural to expect a reasonable graph property
quantifying the information-sharing efficiency to capture this intuition. We have numerically evaluated the effective
independence number of a number of graphs from the above family to test if it satisfies the above criterion. We have
found that the effective independence numbers remain bounded by a \emph{constant} (roughly 30) even when refining the
grid infinitely, confirming that the effective independence number captures the above phenomenon.

Finally, we conducted some numerical simulations to evaluate the average effective independence numbers of certain
types of weighted random graphs. In particular, we considered random graphs with i.i.d.~weights distributed uniformly
on $[0,1]$, $[\frac 12, 1]$ and $[0,\frac 12]$. The distributions of the effective independence numbers are illustrated
as scatter plots for different graph sizes on Figure~\ref{fig:randweights}. First, observe that the average $\alpha^*$
of $U(0,1)$-weighted graphs shows a logarithmic trend in terms of~$N$. The results concerning $U(\frac 12,1)$-weighted
graphs are not surprising given that we have already established that graphs with bounded weights have finite effective
independence numbers. For $U(0,\frac 12)$-weighted graphs, we see that $\alpha^*$ grows linearly up until a certain
threshold when it starts to follow a logarithmic trend. The intuition behind this linear behavior for small graphs is
the following. First, observe that the optimal value of $\varepsilon$ is greater than $1/\sqrt{N}$. That is, until $N$
is large enough so that a critical mass of edges are above this quantity, the optimal value of
$\alpha(\varepsilon)/\varepsilon^2$ remains $N$. Once $N$ is beyond this critical value, $\alpha^*$ starts following a
logarithmic trend.
\section{Analysis}\label{sec:analysis}
Let us now turn to proving Theorem~\ref{thm:mainTheorem}. Our analysis is slightly more general than 
necessary for proving Theorem~\ref{thm:mainTheorem}, in that they allow any sequence of learning rates and IX 
parameters. To avoid clutter, we will omit the~$t$ indices from $s_{t,(\cdot,\cdot)}$. In principle, our analysis 
combines  (more-or-less) standard tools for analyzing \exph with adaptive learning rates and ideas from 
\citet{alon2013from} and \citet{kocak2014efficient}, while also heavily exploiting the structure of our loss 
estimates~\eqref{eq:est}.
In particular, these estimates allow us to bound the expected regret of \expwix in terms of the quantities~$(Q_t)$.
\begin{lemma}
\label{lem:regretbound}
Let $\pa{\eta_t}_t$ and $\pa{\gamma_t}_t$ be two $\pa{\F_t}$-measurable non-increasing sequences satisfying $\gamma_t
\ge \eta_t R$ for all $t$. Then, the expected regret of \expwix is bounded as
\[
R_T \leq \EE{\frac{\log N}{\eta_T} + \sumT \left(\gamma_t + (1+R^2)\eta_t\right)Q_t}.
\]
\end{lemma}
The full proof of the lemma is delegated to the Appendix. Below, we provide a brief sketch covering the key parts of the
proof.
\begin{proof}[Proof sketch]
 By straightforward adaptation of the techniques of
\citet{auer2002finite,bubeck2012regret,gyorfi2007sequential,kocak2014efficient}, we can prove the bound
\[
\begin{split}
\EE{\sum_{t=1}^T \sumi p\ti\hloss\ti} \le& 
\EE{\hL_{T,j} + \frac{\log N}{\eta_T}} 
\\
&+ \EE{\sumT \eta_t \sumi p\ti\left(\hloss\ti\right)^2}.
\end{split}
\]
for any fixed $j$
Thus, we are left with the problem of relating the left-hand side to the total expected loss of the learner and to
upper-bounding the right-hand side.
As the first step, observe that
\begin{align*}
\EEcc{\sumi p\ti\hloss\ti}{\!\F_{t-1}\!} &= \sumi p\ti\frac{\sumj p\tj \gweight^{2}\ji\ell\ti}{\sumj p\tj
\gweight^{2}\ji + \gamma_t}	\\
& \ge \sumi p\ti\ell\ti -
\gamma_t Q_t,
\end{align*}
where we used $\EEcc{\noise\ti}{\F_{t-1}} = 0$ in the first step and $s_{j,i}\le 1$ in the second. The first term on 
the right-hand side can be bounded by $L_{T,j}$ by observing that
$\EEcc{\hloss_{t,j}}{\F_{t-1}}\le \loss_{t,j}$ holds for all fixed $j$ by the definition of the loss
estimates~\eqref{eq:est}. Finally, the last term is bounded as
\begin{align*}
&\EEcc{\sumi p\ti\pa{\hloss\ti}^2}{\!\F_{t-1}\!} \le 
\sumi p\ti\frac{\sumj p\tj \gweight^{2}\ji(1+R^2)}{\left(\sumj p\tj
\gweight^{2}\ji + \gamma_t\right)^2}
\\
&\qquad\qquad\leq \sumi p\ti\frac{1 + R^2}{\left(\sumj p\tj \gweight^{2}\ji + \gamma_t\right)}  = (1+R^2)Q_t.
\end{align*}
The statement of the lemma follows from putting everything together.
\end{proof}
Observe that since we assume $\gweight_{i,i} = 1$ for all $i$, $Q_t$ can be trivially bounded by $N$. As a result, it
is straightforward to show that the regret of \expwix is of order $\sqrt{TN\log N}$.
The remaining challenge is thus bounding~$Q_t$ in a nontrivial way, capturing the structure of the observation graph
$G_t$. The following lemma provides such a bound in terms of the effective independence number of
$G_t$.
\begin{lemma}
\label{lem:qtupperbound}
Let $\alpha^*_t$ be the effective independence number of $G_t$. Then, for any positive $\gamma_t$,
\[
 Q_t \le 2\alpha^*\pa{1 + \log\left(1 + \frac{N^2/\gamma_t+N^2 + N}{\alpha^*}\right)}.
\]
\end{lemma}
\def\argmin{\mathop{\rm arg\, min}}
The proof of this statement builds on results by \citet{alon2013from} and \citet{kocak2014efficient}. Below, we give a
short sketch of the full proof that is given in the Appendix.
\begin{proof}[Proof sketch]
 Let us define $\varepsilon_* = \argmin_{\varepsilon\in[0,1]} \alpha(\varepsilon)/\varepsilon^2$ and observe that
 \[
  \frac{p\ti}{\sum_{j=1}^N p\tj\gweight^2\ji + \gamma_t} \leq 
\frac{1}{\varepsilon^2}\cdot \frac{p\ti}{p\ti + \sum_{j\not =
i}p\tj\II{\gweight\ji\geq\varepsilon} + \gamma_t}
 \]
 holds for all $\varepsilon\in[0,1]$, and in particular for $\varepsilon_*$. Applying a variant
of Lemma~1 in \citet{kocak2014efficient} to the binary graph $G_t(\varepsilon_*)$, we obtain
\[
Q_t\leq \frac{2}{\varepsilon_*^2}\cdot\left(\alpha_t(\varepsilon_*)\log\left(1 + \frac{\varepsilon_*^2
N^2/\gamma_t+N + 1}{\alpha_t(\varepsilon_*)}\right) + \frac{2}{\varepsilon_*^2}\right).
\]
The statement of the lemma follows from using the trivial bound $\alpha_t(\varepsilon_*)\ge 1$.
\end{proof}
Now, every ingredient is ready for proving Theorem~\ref{thm:mainTheorem}. In particular, 
plugging in the choice of the parameters~$\eta_t$ and $\gamma_t$ into the bound of Lemma~\ref{lem:regretbound} and
applying Lemma~3.5 of \citet{auer2002adaptive}, we obtain
 \[
R_T   \le 2\sqrt{2(1+R+R^2)\left(N+\sum_{t = 1}^{T}Q_t\right) \log N}.
 \]
Then, the statement of the theorem follows from combining the above with the
bound of Lemma~\ref{lem:qtupperbound}.

\begin{figure*}[t]
	\centering
	\includegraphics[width = .3\textwidth]{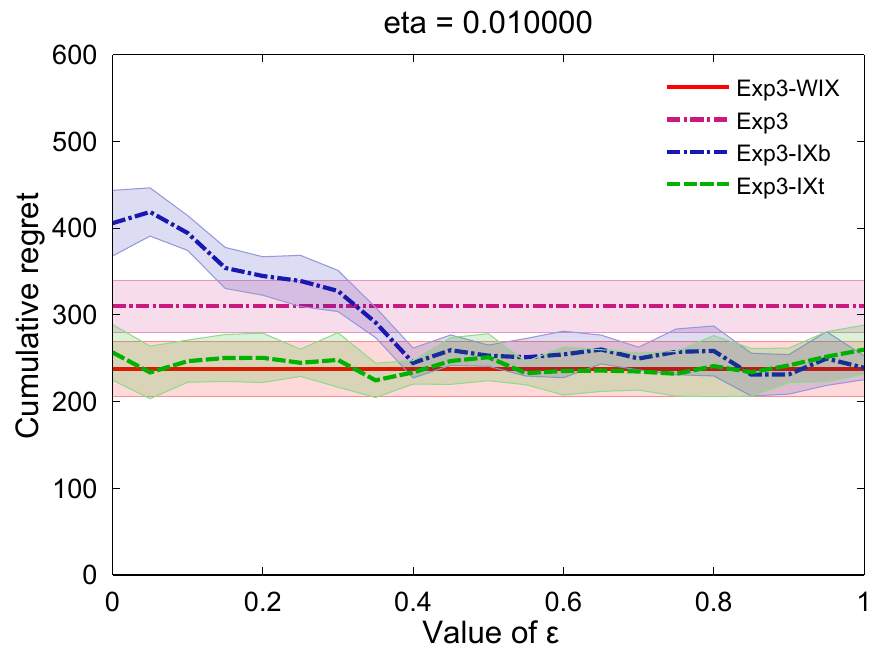}
	\includegraphics[width = .3\textwidth]{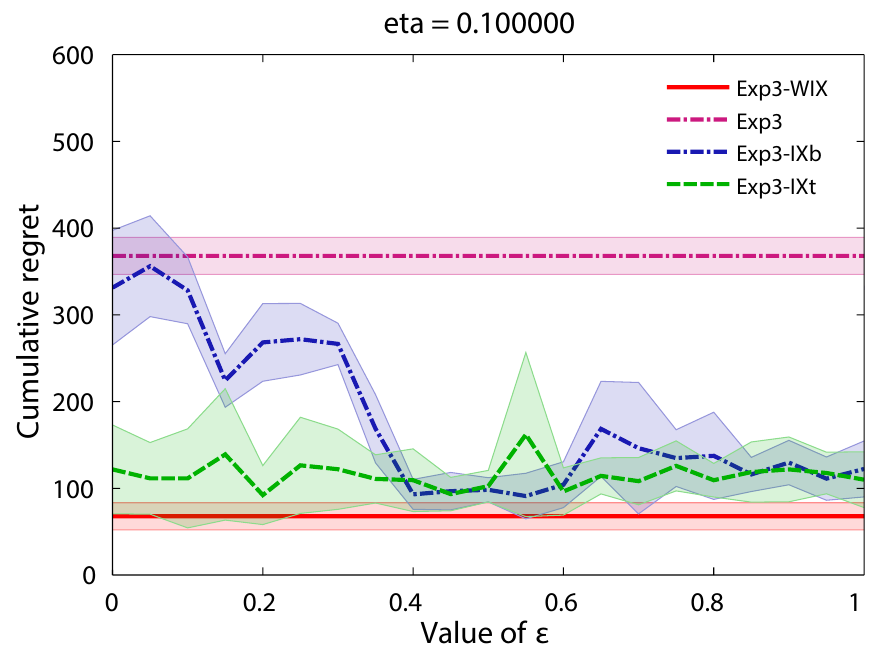}
	\includegraphics[width = .3\textwidth]{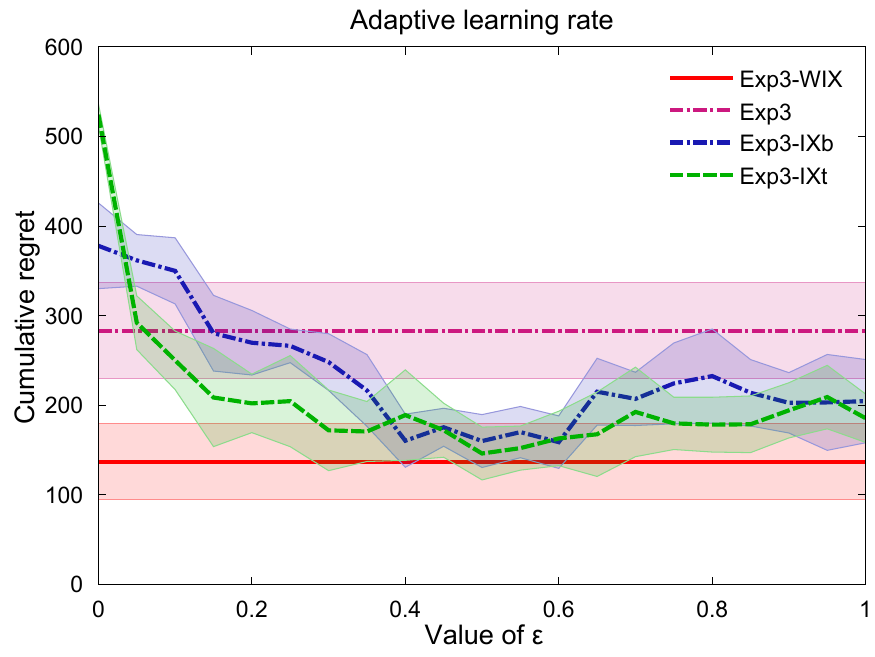}
	\caption{Comparison  of total regrets of the algorithms at time $T$ for static and adaptive learning rates.}
	\label{fig:comparison}
\end{figure*}
\section{Experiments}
In this section, we empirically compare \expwix{} to some of its natural competitors: \expixt, vanilla \exph{} that
ignores all side observations and a straightforward variation of the \exphix{} algorithm of \citet{kocak2014efficient}.
This latter algorithm, referred to as \expixb{} (with ``{\sc b}'' standing for ``basic''), uses a threshold
$\varepsilon$ to decide which observations are too noisy to use and which are the ones to be retained: All the edges
with weights smaller than a parameter $\varepsilon$ are deleted and the rest of the weights are set to $1$. The
algorithm then plays basic \exphix for the resulting binary graph. That is, the difference between \expixt and 
\expixb is that the latter does not adjust for the bias arising from using unreliable side observations.
Note that \expixb comes without any formal performance guarantee.

For the purpose of the experiments, we  assumed to have 25 actions forming $5\times5$ grid  embedded in a plane. The 
distance of neighbors in the grid was set to be 1. Using this structure, we defined the weight connecting two nodes as 
$\min\ev{3/d^2,1}$, and $d$ is the Euclidean distance between actions in the grid. 

For constructing the loss sequence, we interleaved 20 Gaussian random walks with small increments for each action, with appropriate 
truncations to keep the losses within the $[0,1]$ interval.
 Using this procedure, we generated a single loss sequence of $T=5,000$ steps to test the algorithms. 
For a fair comparison, we ran each algorithm for their respective theoretically motivated adaptive 
learning rates, and also for a number of static learning rates. For static learning rates, we observed the best performance of \exph for 
learning rates 
around $0.01$, all the other algorithms did well for learning rates around $0.1$. Due to the lack of space, we included 
plots only for these two learning rates.

We ran \expixb{} and \expixt{} for several values of $\varepsilon$ from $0$ to 
$1$. In all experiments, we set the implicit exploration parameters to zero. This is well-justified in the case of 
undirected graphs, as shown by the analysis of \citet{alon2013from}.
Figure~\ref{fig:comparison} shows the performance of the algorithms as a function of the threshold parameter 
$\varepsilon$.
Each curve on this graph is the average of the total regrets measured in 10 independent runs with error bars 
proportional to the empirical standard deviation. 

Our experiments confirm that guessing the right value for the threshold parameter is indeed a very difficult problem: 
while \expwix{} performs consistently well for all parameter settings, \expixt{} and \expixb{} only perform reasonably 
well for moderate values of~$\varepsilon$ that are not supported by theory. (In fact, the graph is designed so that the value of 
$\varepsilon$ optimizing $\alpha(\varepsilon)/\varepsilon^2$ is~1, which suggests that only observations from immediate neighbors
in the grid are relevant.) Perhaps 
surprisingly, \expixb performs well despite the obvious bias in its loss estimates. The performance of \exph is 
significantly worse than \expwix, confirming the benefit of side-observations, however noisy they are.

\vspace{-.25cm}
\section{Conclusions and open problems}
\vspace{-.25cm}
The main contribution of our work is introducing a new partial-observability model for adversarial online learning and 
proposing an efficient learning algorithm with rigorous performance guarantees for this setting. Our regret bounds 
depend on a newly introduced graph property that we call the effective independence number. While the recent results of 
\citet{wu2015online} suggest that our bounds are minimax optimal in some special cases of our framework, it is not yet known 
whether the effective independence number is the exact quantity that characterizes the minimax regret in general---we 
leave this exciting question open for future investigation. 

\paragraph{Acknowledgements}
\label{sec:Acknowledgements}
The research presented in this paper was supported by French Ministry of
Higher Education and Research, Nord-Pas-de-Calais Regional Council,  French National Research Agency project ExTra-Learn (n.ANR-14-CE24-0010-01), 
and by UPFellows Fellowship (Marie Curie COFUND program n${^\circ}$ 600387).

\newpage
{\fontsize{10}{11.75}\selectfont
\bibliography{library,library2,allbib}
\bibliographystyle{icml2015}
}

\appendix
\onecolumn
\section{Proof of Lemma \ref{lem:regretbound}}
The first part of the analysis is similar to the analysis of the basic {\sc Exp3} algorithm besides, we are using an adaptive learning rate $\eta_t$ to obtain anytime regret bound, and therefore, we do not need to know the stopping time $T$ of the algorithm. We start by introducing some notation. Let
\[
W_{t} = \frac 1N\sumi e^{-\eta_{t}\hL_{t-1,i}} \qquad\mbox{and}\qquad
W'_{t} = \frac 1N\sumi e^{-\eta_{t-1}\hL_{t-1,i}}.
\]
Following the proof of Lemma~1 of \cite{gyorfi2007sequential}, we track the evolution of $\log (W_{t+1}'/W_t)$ to control the regret. We have
\begin{align}
\nonumber\frac{1}{\eta_t}\log\frac{W'_{t+1}}{W_t} &=\frac{1}{\eta_t}\log
\sumi\frac{\frac 1Ne^{-\eta_{t}\hL_{t,i}}}{W_t}
=\frac{1}{\eta_t}\log\sumi\frac{\frac 1Ne^{-\eta_t\hL_{t-1,i}} e^{-\eta_t\hloss\ti}}{W_t}
			\\
\label{eq:tailor}&=\frac{1}{\eta_t}\log\sumi p\ti e^{-\eta_t\hloss\ti}	\leq
\frac{1}{\eta_t}\log\sumi p\ti\left(1-\eta_t\hloss\ti+
(\eta_t\hloss\ti)^2\right)	\\
\nonumber&=\frac{1}{\eta_t}\log\left(1-\eta_t\sumi p\ti\hloss\ti
+\eta_t^2\sumi p\ti(\hloss\ti)^2\right),
\end{align}
where in (\ref{eq:tailor}), we used the inequality $\exp(-x) \le 1 - x + x^2$ that holds for 
$x\ge -1$. The use of this inequality is made possible by the definitions of $\eta_t$ and $\gamma_t$ that guarantee
$\eta_t\hloss\ti\ge-1$ for all $i$. Further, we 
use the inequality $\log(1-x)\leq-x$, which holds for all $x$, to upper bound last term
\begin{align}
\sumi p\ti \hloss\ti&\leq\left[\frac{\log
\nonumber W_t}{\eta_t} -\frac{\log W'_{t+1}}{\eta_t}\right]+\sumi\eta_t p\ti(\hloss\ti)^2	\\
\label{eq:eq2}&=\left[\left(\frac{\log
W_t}{\eta_t} 
-\frac{\log W_{t+1}}{\eta_{t+1}}\right)+\left(\frac{\log W_{t+1}}{\eta_{t+1}}
-\frac{\log 
W'_{t+1}}{\eta_t}\right)\right]+\sumi\eta_t p\ti(\hloss\ti)^2.	
\end{align}

The second term in brackets on the right hand side is upper bounded by zero, since
$$W_{t+1} = \sumi\frac{1}{N}e^{-\eta_{t+1}\hL_{t,i}} =
\sumi\frac{1}{N}\left(e^{-\eta_{t}\hL_{t,i}}\right)^{\frac{\eta_{t+1}}{
\eta_t}}\leq\left(\sumi\frac{1}{N}e^{
-\eta_{t}\hL_{t,i}}\right)^{\frac{\eta_{t+1}}{\eta_t}} =
(W'_{t+1})^{\frac{\eta_{t+1}}{\eta_t}}.$$
Using Jensen's inequality to the concave function
$x^{\frac{\eta_{t+1}}{\eta_t}}$ for $x\in \R$. The function is
concave since $\eta_{t+1}\leq\eta_t$ by definition. Taking logarithms in the
above inequality, we get

$$\frac{\log W_{t+1}}{\eta_{t+1}}-\frac{\log W'_{t+1}}{\eta_t}\leq 0.$$

Using this inequality, we can simplify (\ref{eq:eq2})
\[
\sumi p\ti\hloss\ti\leq\eta_t\sumi p\ti\left(\hloss\ti\right)^2+\left(\frac{\log W_t}{\eta_t}-\frac{\log
W_{t+1}}{\eta_{t+1}}\right).
\]
Taking \emph{conditional} expectations with respect to the $\sigma$-algebra $\F_{t-1}$, generated by the history up to time $t-1$, and summing up both sides over the time,
we get
\begin{align}\label{eq:genbound}
\sumT\EEcc{\sumi p\ti\hloss\ti}{\!\F_{t-1}\!}\!\leq\sumT\EEcc{\eta_t
\sumi p\ti\left(\hloss\ti\right)^2}{\!\F_{t-1}\!}
\!+\!\sumT\EEcc{\frac{\log W_t}{\eta_t}\!-\!\frac{\log
W_{t+1}}{\eta_{t+1}}}{\!\F_{t-1}\!}\!\!.
\end{align}

For the following part of the analysis, we use a slightly more general form of our loss estimates,
$$\hloss\ti = \sumi p\ti\frac{\gweight^\delta\Ii\oloss\ti}{\sumj p\tj \gweight^{1+\delta}\ji + \gamma_t} = \sumi
p\ti\frac{\gweight^{1+\delta}\Ii\ell\ti + \gweight\Ii^\delta(1-\gweight\Ii)\noise\ti}{\sumj p\tj \gweight^{1+\delta}\ji
+ \gamma_t}.$$
Later we show that $\delta = 1$ is optimal, which recovers the loss estimates \eqref{eq:est}.
The next step is to bound the three expectations involved in Equation~\eqref{eq:genbound}. For the first expectation we
have 
\begin{align*}
\EEcc{\sumi p\ti\hloss\ti}{\!\F_{t-1}\!} &= \EEcc{\sumi p\ti\frac{\gweight^\delta\Ii\oloss\ti}{\sumj p\tj \gweight^{1+\delta}\ji + \gamma_t}}{\!\F_{t-1}\!}= \sumi p\ti\frac{\sumj p\tj \gweight^{1+\delta}\ji\ell\ti}{\sumj p\tj \gweight^{1+\delta}\ji + \gamma_t}	\\
& \geq \sumi p\ti\ell\ti - \gamma_t\sumi\frac{p\ti}{\sumj p\tj \gweight^{1+\delta}\ji + \gamma_t} = \sumi p\ti\ell\ti - \gamma_t Q_t(\delta),
\end{align*}
where $$Q_t(\delta) = \sumi\frac{p\ti}{\sumj p\tj \gweight^{1+\delta}\ji + \gamma_t}\cdot$$
For the second expectation we have
\begin{align*}
\EEcc{\sumi p\ti(\hloss\ti)^2}{\!\F_{t-1}\!} &= \sumi p\ti\frac{\EEcc{ \gweight^{2+2\delta}\Ii }{\!\F_{t-1}\!}\ell^2\ti + \EEcc{\gweight\Ii^{2\delta}(1-\gweight\Ii)^2}{\!\F_{t-1}\!}\EEcc{\noise\ti^2}{\!\F_{t-1}\!}}{\left(\sumj p\tj \gweight^{1+\delta}\ji + \gamma_t\right)^2} \\
&\leq \sumi p\ti\frac{\sumj p\tj \gweight^{2+2\delta}\ji + \sumj p\tj \gweight^{2\delta}\ji R^2}{\left(\sumj p\tj \gweight^{1+\delta}\ji + \gamma_t\right)^2} \\
&\leq \sumi p\ti\frac{1 + R^2}{\left(\sumj p\tj \gweight^{1+\delta}\ji + \gamma_t\right)}  = (1+R^2)Q_t(\delta).
\end{align*}
Where the last inequality holds for $\delta\geq1$. 
For the third expectation we have
\begin{align*}
-\EE{\frac{\log W_{T+1}}{\eta_{T+1}}} \leq \min_{k\in [N]}\left(-\EE{\frac{\log
\frac{1}{N}e^{-\eta_T\hL_{T,k}}}{\eta_T}}\right)  = \EE{\frac{\log N}{\eta_T}} +
\min_{k\in[N]}\left(\EE{\hL_{T,k}}\right).
\end{align*}

To conclude, observe that $Q_t(1)\le Q_t(d)$ holds almost surely and thus we can set $\delta = 1$. Then the statement
of the lemma follows from combining all of the bounds above.
\qed

\section{Proof of Lemma \ref{lem:qtupperbound}}
 As done in the analysis of \citet{alon2013from,kocak2014efficient}, we use following two lemmas to bound $Q_t$.
\begin{lemma}{\rm (cf.~Lemma 10 of \cite{alon2013from})}\label{lem:indBoundForGraph}
~Let $G$ be a directed graph, with $V = \{1,\dots,N\}$. Let $d_i^-$ 
be
the indegree of the node $i$ and $\alpha =
\alpha(G)$ be the independence number of $G$. Then
$$\sum_{i=1}^N\frac{1}{1+d_i^-}\leq2\alpha\log\left(1+\frac{N}{\alpha}
\right).$$
\end{lemma}


\begin{lemma}{\rm (cf.~Lemma 12 of \cite{alon2013from})}\label{lem:techLemma}
~If $a,\,b\geq 0$ and $a+b\geq B>A>0$, then
$$\frac{a}{a+b-A}\leq\frac{a}{a+b}+\frac{A}{B-A}\cdot$$
\end{lemma}


Before using the previous lemmas, we need to discretize the values of $p\ti$. Let $\hat p\ti$ be the discretized
version of $p\ti$ which satisfies $\hp\ti = k/M$ for some integer $k$, $M = \lceil \varepsilon^2 N^2/\gamma_t\rceil$,
and $\hp\ti -1 < p\ti \leq \hp\ti$. 
\begin{center}
\includegraphics[width = 4in]{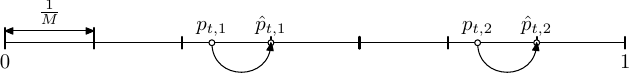}
\end{center}
Using Lemma \ref{lem:techLemma} for $a = \hp\ti$, $b = \sum_{j\not = i} \hp\tj\mathbb{I}\{\gweight\ji\geq\varepsilon\} + \gamma/\varepsilon$, $B = \gamma/\varepsilon$, and $A = N/M$ we get
\begin{align*}
\frac{p\ti}{p\ti + \sum_{j\not = i}p\tj\gweight^2\ji + \gamma_t} &\leq \frac{p\ti}{\varepsilon^2 p\ti + \sum_{j\not = i}\varepsilon^2 p\tj\mathbb{I}\{\gweight\ji\geq\varepsilon\} + \gamma_t}\\
&= \frac{1}{\varepsilon^2}\frac{p\ti}{p\ti + \sum_{j\not = i}p\tj\mathbb{I}\{\gweight\ji\geq\varepsilon\} + \gamma_t/\varepsilon^2}\\
&\leq \frac{1}{\varepsilon^2}\frac{\hp\ti}{\hp\ti + \sum_{j\not = i}\hp\tj\mathbb{I}\{\gweight\ji\geq\varepsilon\} + \gamma_t/\varepsilon^2-N/M}\\
&\leq \frac{1}{\varepsilon^2}\left(\frac{\hp\ti}{\hp\ti + \sum_{j\not = i}\hp\tj\mathbb{I}\{\gweight\ji\geq\varepsilon\}} + \frac{N/M}{\gamma_t/\varepsilon^2 - N/M}\right)\\
&\leq \frac{1}{\varepsilon^2}\left(\frac{\hp\ti}{\hp\ti + \sum_{j\not = i}\hp\tj\mathbb{I}\{\gweight\ji\geq\varepsilon\}} + \frac{2}{N}\right).\\
\end{align*}
From this point, one can follow the proof of Lemma~1 by \citet{kocak2014efficient} to prove
\begin{align*}
Q_t &\leq\frac{1}{\varepsilon^2}\sumi\frac{\hp\ti}{\hp\ti + \sum_{j\not =
i}\hp\tj\mathbb{I}\{\gweight\ji\geq\varepsilon\}} +
\frac{2}{\varepsilon^2}\leq\frac{2}{\varepsilon^2}\alpha_t(\varepsilon)\log\left(1 +
\frac{M+N}{\alpha_t(\varepsilon)}\right) + \frac{2}{\varepsilon^2}\\
&\leq\frac{2}{\varepsilon^2}\alpha_t(\varepsilon)\log\left(1 + \frac{ N^2/\gamma_t+N/\varepsilon^2 +
1/\varepsilon^2}{\alpha_t(\varepsilon)/\varepsilon^2}\right) + \frac{2}{\varepsilon^2}.
\end{align*}
This bound holds for every $\varepsilon\in [0,1]$, therefore, it holds also for $\varepsilon_*$. Finally, using
$1/\varepsilon^2\leq N$ and $\alpha(\varepsilon)\geq 1$, we can recover the statement of the lemma. \qed

\section{The proof of Theorem~\ref{thm:naive}}
The proof of Theorem~\ref{thm:naive} roughly follows the proof of Theorem~\ref{thm:mainTheorem} with one key 
difference. For simplicity, let us define 
\[
 Q'_t = \sumi\frac{p\ti}{\sumj p\tj \gweight\ji\II{\gweight\ji\ge\varepsilon_t} + \gamma_t}
\]
and consider an oblivious adversary that chooses the whole sequence of observations graphs deterministically before the 
first round.
Our starting point is the bound of Equation~\eqref{eq:genbound}, which also 
holds for \expixt. First, we have
\begin{align*}
\EEcc{\sumi p\ti\hloss\ti}{\!\F_{t-1}\!} &= \EEcc{\sumi p\ti\frac{\oloss\ti \II{\gweight\Ii\ge\varepsilon_t}} 
{\sumj p\tj \gweight\ji\II{\gweight\ji\ge\varepsilon_t} + \gamma_t}}{\!\F_{t-1}\!}= \sumi p\ti\frac{\sumj p\tj 
\gweight\ji\II{\gweight\ji\ge\varepsilon_t}\ell\ti}{\sumj p\tj \gweight\ji\II{\gweight\ji\ge\varepsilon_t} + \gamma_t}	
\\
& \geq \sumi p\ti\ell\ti - \gamma_t\sumi\frac{p\ti}{\sumj p\tj \gweight\ji\II{\gweight\ji\ge\varepsilon_t} + \gamma_t} 
= \sumi p\ti\ell\ti - \gamma_t Q_t'.
\end{align*}
Furthermore, 
\begin{align*}
\EEcc{\sumi p\ti(\hloss\ti)^2}{\!\F_{t-1}\!} &= \sumi p\ti\frac{\EEcc{ \gweight^2\Ii 
\II{\gweight\Ii\ge\varepsilon_t} }{\!\F_{t-1}\!}\ell^2\ti + 
\EEcc{\pa{1-\gweight\Ii\II{\gweight\Ii\ge\varepsilon_t}}^2}{\!\F_{t-1}\!}\EEcc{\noise\ti^2}{\!\F_{t-1}\!}}{\left(\sumj 
p\tj \gweight\ji\II{\gweight\ji\ge\varepsilon_t} + \gamma_t\right)^2} \\
&\leq \sumi p\ti\frac{\sumj p\tj \gweight^2\ji + R^2}{\left(\sumj p\tj 
\gweight\ji\II{\gweight\ji\ge\varepsilon_t} + \gamma_t\right)^2} 
\leq \frac{1}{\varepsilon_t} \sumi p\ti\frac{1 + R^2}{\sumj p\tj 
\gweight\ji\II{\gweight\ji\ge\varepsilon_t} + \gamma_t} 
\\
&= \frac{(1+R^2)}{\varepsilon_t} Q_t',
\end{align*}
where the last inequality uses that $\sumj p\tj \gweight\ji\II{\gweight\ji\ge\varepsilon_t} + \gamma_t \ge 
\varepsilon_t$. 
Now, following the proof of Lemma~\ref{lem:qtupperbound}, we can prove 
\[
Q_t' \le 2\frac{\alpha(G_t(\varepsilon_t))}{\varepsilon_t}\pa{1 + \log\left(1 + \frac{N^2/\gamma_t+N + 
1}{\alpha}\right)}.
\]

For finishing the proof, let us set $\eta_t=\eta\ge 0$ and $\gamma_t=\gamma\ge 0$ for all $t$.
Putting all of the above results together, we get
\[
\begin{split}
 R_T &\le \frac{\log N}{\eta} + \gamma \sum_{t=1}^T Q'_t + \eta (1+R^2)\sum_{t=1}^T \frac{Q'_t}{\varepsilon_t}
 \\
 &\le \frac{\log N}{\eta} + \gamma C_1 \sum_{t=1}^T \frac{\alpha(G_t(\varepsilon_t))}{\varepsilon_t} + \eta C_2
(1+R^2)\sum_{t=1}^T \frac{\alpha(G_t(\varepsilon_t))}{\varepsilon_t^2},
\end{split}
\]
where $C_1$ and $C_2$ are $\OO\pa{\log \pa{N/\gamma}}$. Optimizing the choice of $\eta$ and $\gamma$ concludes the 
proof of Theorem~\ref{thm:naive}. \qed

\end{document}